\documentclass[sigconf,authorversion]{acmart}

\usepackage{graphicx}
\usepackage{subfig}
\usepackage{amsmath}
\usepackage{makecell}
\usepackage{microtype} 
\usepackage{booktabs}
\usepackage{amsfonts} 
\usepackage{nicefrac}

\theoremstyle{definition}
\newtheorem{definition}{Definition}
\newtheorem{theorem}{Theorem}

\newcommand{\norm}[1]{\left\lVert#1\right\rVert}

\AtBeginDocument{%
  \providecommand\BibTeX{{%
    \normalfont B\kern-0.5em{\scshape i\kern-0.25em b}\kern-0.8em\TeX}}}

\copyrightyear{2020}
\acmYear{2020}
\setcopyright{acmcopyright}\acmConference[ETRA '20 Short Papers]{Symposium on Eye Tracking Research and Applications}{June 2--5, 2020}{Stuttgart, Germany}
\acmBooktitle{Symposium on Eye Tracking Research and Applications (ETRA '20 Short Papers), June 2--5, 2020, Stuttgart, Germany}
\acmPrice{15.00}
\acmDOI{10.1145/3379156.3391364}
\acmISBN{978-1-4503-7134-6/20/06}

\begin{document}

\title{Privacy Preserving Gaze Estimation using Synthetic Images via a Randomized Encoding Based Framework}

\author{Efe Bozkir}
\authornote{Both authors contributed equally to this research.}
\email{efe.bozkir@uni-tuebingen.de}
\affiliation{%
  \department{Human-Computer Interaction}
  \institution{University of T{\"u}bingen, Germany}
}

\author{Ali Burak {\"U}nal}
\authornotemark[1]
\email{uenal@informatik.uni-tuebingen.de}
\affiliation{%
  \department{Methods in Medical Informatics}
  \institution{University of T{\"u}bingen, Germany}
}

\author{Mete Akg{\"u}n}
\email{mete.akguen@uni-tuebingen.de}
\affiliation{%
  \department{Methods in Medical Informatics, Translational Bioinformatics}
  \institution{University of T{\"u}bingen, Germany}
}

\author{Enkelejda Kasneci}
\email{enkelejda.kasneci@uni-tuebingen.de}
\affiliation{%
  \department{Human-Computer Interaction}
  \institution{University of T{\"u}bingen, Germany}
}

\author{Nico Pfeifer}
\authornote{Also affiliated with Statistical Learning in Computational Biology, Max Planck Institute for Informatics, Saarbr\"ucken, Germany}
\email{pfeifer@informatik.uni-tuebingen.de}
\affiliation{%
  \department{Methods in Medical Informatics}
  \institution{University of T{\"u}bingen, Germany}
}

\renewcommand{\shortauthors}{Bozkir and {\"U}nal, et al.}

\begin{abstract}
Eye tracking is handled as one of the key technologies for applications that assess and evaluate human attention, behavior, and biometrics, especially using gaze, pupillary, and blink behaviors. One of the challenges with regard to the social acceptance of eye tracking technology is however the preserving of sensitive and personal information. To tackle this challenge, we employ a privacy-preserving framework based on randomized encoding to train a Support Vector Regression model using synthetic eye images privately to estimate the human gaze. During the computation, none of the parties learn about the data or the result that any other party has. Furthermore, the party that trains the model cannot reconstruct pupil, blinks or visual scanpath. The experimental results show that our privacy-preserving framework is capable of working in real-time, with the same accuracy as compared to non-private version and could be extended to other eye tracking related problems.
\end{abstract}



\maketitle

\section{Introduction}
Recent advances in the fields of Head-Mounted-Display (HMD) technology, computer graphics, augmented reality (AR), and eye tracking enabled numerous novel applications. One of the most natural and non-intrusive ways of interaction with HMDs or smart glasses is achieved by gaze-aware interfaces using eye tracking. However, it is possible to derive a lot of sensitive and personal information from eye tracking data such as intentions, behaviors, or fatigue since eyes are not fully controlled in a conscious way.

It has been shown that cognitive load \cite{Chen2014UsingTP,Appel2018}, visual attention \cite{Bozkir:2019:ADA:3343036.3343138}, stress \cite{kubler2014stress}, task identification \cite{Borji2014}, skill level assessment and expertise \cite{Liu2009,eivazi2017optimal,castner2018scanpath}, human activities \cite{Steil:2015:DEH:2750858.2807520,braunagel2017online}, biometric information and authentication \cite{Kinnunen:2010:TTP:1743666.1743712,Komogortsev2010,6712725,Zhang:2018:CAU:3178157.3161410,Abdrabou:2019:JGW:3314111.3319837}, or personality traits \cite{Berkovsky:2019:DPT:3290605.3300451} can be obtained using eye tracking data. Since highly sensitive information can be derived from eye tracking data, it is not surprising that HMDs or smart glasses have not been adopted by large communities yet. According to a recent survey \cite{Steil2019}, people agree to share their eye tracking data only when it is co-owned by a governmental health-agency or is used for research purposes. This indicates that people are hesitant about sharing their eye tracking data in commercial applications. Therefore, there is a likelihood that larger communities could adopt HMDs or smart glasses if privacy-preserving techniques are applied in the eye tracking applications. The reasons why privacy preserving schemes are needed for eye tracking are discussed in \cite{Liebling2014} extensively. However, until now, there are not many studies in privacy-preserving eye tracking. Recently, a method to detect privacy sensitive everyday situations \cite{Steil:2019:PPH:3314111.3319913}, an approach to degrade iris authentication while keeping the gaze tracking utility in an acceptable accuracy \cite{John:2019:EDI:3314111.3319816}, and differential privacy based techniques to protect personal information on heatmaps and eye movements \cite{Liu2019,Steil2019} are introduced. While differential privacy can be applied to eye tracking data for various tasks, it introduces additional noise on the data which causes decrease in the utility \cite{Liu2019,Steil2019}, and it might lead to less accurate results in computer vision tasks, such as gaze estimation or activity recognition.

In light of the above, function-specific privacy models are required. In this work, we focus on the gaze estimation problem as a proof-of-concept by using synthetic data including eye landmarks and ground truth gaze vectors. However, the same privacy-preserving approach can be extended to any feature-based, eye tracking problem such as intention, fatigue, or activity detection, in HMD or unconstrained setups due to the demonstrated real-time working capabilities. In our study, the gaze estimation task is solved by using Support Vector Regression (SVR) models in a privacy-preserving manner by computing the dot product of eye landmark vectors to obtain the kernel matrix of the SVR for a scenario, where two parties have the eye landmark data, each of which we call {\it input-party}, and one {\it function-party} that trains a prediction model on the data of the input-parties. This scenario is relevant when the input-parties use eye tracking data to improve the accuracy of their models and do not share the data due to the privacy concerns. To this end, we utilize a framework employing randomized encoding \cite{unal2019framework}. In the computation, neither the eye images nor the extracted features are revealed to the function-party directly. Furthermore, the input-parties do not infer the raw eye tracking data or result of the computation. Eye images that are used for training and testing are rendered using UnityEyes \cite{wood2016_etra} synthetically and 36 landmark-based features \cite{Park2018} are used. To the best of our knowledge, this is the first work that applies a privacy-preserving scheme based on function-specific privacy models on an eye tracking problem.

\section{Threat Model}
We assume that the input-parties are semi-honest (honest but curious) that are not allowed to deviate from the protocol description while they try to infer some valuable information about other parties' private inputs using their views of the protocol execution. We also assume that the function-party is malicious and the input-parties and the function-party do not collude. 

\section{Methodology}
In this section, we discuss the data generation, randomized encoding, and privacy-preserving gaze estimation framework.

\subsection{Data Generation}
To train and evaluate the gaze estimator, we generate eye images and gaze vectors. As our work is a proof-of-concept and requires high amount of data, synthetic images from UnityEyes \cite{wood2016_etra}, which is based on the Unity3D, are used. \textit{Camera parameters} and \textit{Eye parameters} are chosen as $(0,0,0,0)$ (fixed camera) and $(0,0,30,30)$ (eyeball pose range parameters in degrees), respectively. $20,000$ images are rendered in \textit{Fantastic} quality setting and $512 \times 384$ screen resolution. Then, processing and normalization pipeline from \cite{Park2018} is employed. In the end, we obtain $128 \times 96$ sized eye images, 18 eye landmarks including eight iris edge, eight eyelid, one iris center, and one iris-center-eyeball-center vector normalized according to Euclidean distance between eye corners, and gaze vectors using pitch and yaw angles. Final feature vectors consist of $36$ elements. Figure \ref{fig:eyeImages} shows an example illustration.

\begin{figure}
  \centering
   \subfloat[Landmarks.]{{\includegraphics[scale=0.3246]{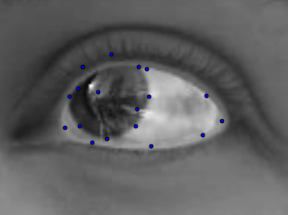}}}
   \qquad
   \qquad
   \subfloat[Gaze.]{{\includegraphics[scale=0.3246]{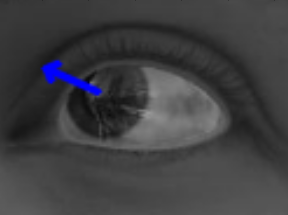} }}%
  \caption{Eye landmarks and gaze on a synthetic image.}
  \label{fig:eyeImages}
\end{figure}

\subsection{Randomized Encoding}
The utilized framework employs randomized encoding (RE) \cite{applebaum2006cryptography,applebaum2006computationally} to compute the  dot product of the landmark vectors. The dot product is needed to compute kernel matrix of the SVR which is later used for training the gaze estimator and validation of the framework.

In the randomized encoding, the computation of a function $f(x)$ is performed by a randomized function $\hat{f}(x;r)$ where $x$ is the input value, which corresponds to eye landmarks in our setup, and $r$ is the random value. The idea is to encode the original function by using random value(s) such that the combination of the components of the encoding reveals only the output of the original function. In the framework, the computation of the dot product is accomplished by utilizing the decomposable and affine randomized encoding (DARE) of addition and multiplication \cite{applebaum2017garbled}. The encoding of multiplication is as follows.

\begin{definition}[Perfect RE for Multiplication \cite{applebaum2017garbled}]
\label{def:remul}
A multiplication function is defined as $f_m(x_1,x_2)=x_1 \cdot x_2$ over a ring $\mathsf{R}$. One can perfectly encode the $f_m$ by employing the DARE $\hat{f_m}(x_1,x_2;r_1,r_2,r_3)$:
\begin{equation*}
\begin{aligned}
    \hat{f_m}(x_1,x_2;r_1,r_2,r_3) = ( & x_1+r_1, x_2+r_2, \\
    & r_2x_1+r_3, r_1x_2+r_1r_2-r_3),
\end{aligned}
\end{equation*}
where $r_1,r_2$ and $r_3$ are uniformly chosen random values. The recovery of $f_m(x_1,x_2)$ can be accomplished by computing $c_1 \cdot c_2 - c_3 - c_4$ where $c_1=x_1+r_1$, $c_2=x_2+r_2$, $c_3=r_2x_1+r_3$ and $c_4=r_1x_2+r_1r_2-r_3$. The simulation of $\hat{f_m}$ can be done perfectly by the simulator $\mathsf{Sim}(y;a_1,a_2,a_3) := (a_1,a_2,a_3,a_1a_2-y-a_3)$ where $a_1, a_2$ $a_3$ are random values.
\end{definition}

\subsection{Framework}
To perform the private gaze estimation task in our scenario, we inspire from the framework as in \cite{unal2019framework} due to its efficiency compared to other approaches in the literature. The framework is proposed to compute the addition or multiplication of the input values of two input-parties in the function-party by utilizing randomized encoding. We utilize the multiplication operation over the eye landmark vectors to compute the dot product of these vectors to obtain kernel matrix of the SVR in a privacy-preserving way.

We have two input-parties as Alice and Bob, having the eye landmark data as $X \in \mathbb{R}^{n_f \times n_a}$ and $Y \in \mathbb{R}^{n_f \times n_b}$ where $n_a$ and $n_b$ represent the number of samples in Alice and Bob, respectively, and $n_f$ is the number of features. In addition to the input-parties, there exists a server that trains a model on the data of the input-parties. $A_{.j}$ for any matrix $A$ represents the $j$-th column of the corresponding matrix and ''$\odot$`` represents the element-wise multiplication of the vectors. As a first step, Alice creates a uniformly chosen random value $r_3 \in \mathbb{R}$ and two vectors $r_1, r_2 \in \mathbb{R}^{n_f}$ with uniformly chosen random values, which are used to encode the element-wise multiplication of the vectors and shares them with Bob. Afterwards, Bob computes $C^2_{.j}=Y_{.j}+r_2$ and $C^4_j= \sum_{d=1}^{n_f} (r_1 \odot Y_{.j} + r_1 \odot r_2)_d - r_3$,  $\forall j \in \{1,\cdots,n_b\}$ where $C^2 \in \mathbb{R}^{n_f \times n_b}$ and $C^4 \in \mathbb{R}^{n_b}$. Meanwhile, Alice computes $C^1_{.i}=X_{.i}+r_1$ and $C^3_{i}=\sum_{d=1}^{n_f}(r_2 \odot X_{.i})_d + r_3$, $\forall i \in \{1,\cdots,n_a\}$ where $C^1 \in \mathbb{R}^{n_f \times n_a}$ and $C^3 \in \mathbb{R}^{n_a}$. Input-parties send their share of the encoding to the server with the gram matrix of their samples, which is the dot product among their samples. Then, the server computes the dot product between samples of Alice and Bob to complete the missing part of the gram matrix of all samples. To achieve this, the server computes $k_{ij} = \sum_{d=1}^{n_f}(C^1_{.i} \odot C^2_{.j})_d - C^3_{i} - C^4_{j}$, $\forall i \in \{1,\cdots,n_a\}$ and $\forall j \in \{1,\cdots,n_b\}$ where $k_{ij}$ is the $i$-th row $j$-th column entry of the gram matrix between the samples of the input-parties. Once the server has all components of the gram matrix, it constructs the complete gram matrix $K$ by simply concatenating the parts of it.  In our solution, Alice and Bob send to the server $(C^1, C^3)$ and $(C^2, C^4)$ tuples, respectively. These components reveal nothing but only the gram matrix of the samples after decoding. Furthermore, the input-parties shuffle their raw data before the computation to avoid the possibility of private information leakage such as the behavior of the person due to the nature of the visual sequence information. The overall flow is summarized in Figure \ref{fig:FlowOfArchitecture}.

\begin{figure}
  \centering
   \includegraphics[scale=0.58]{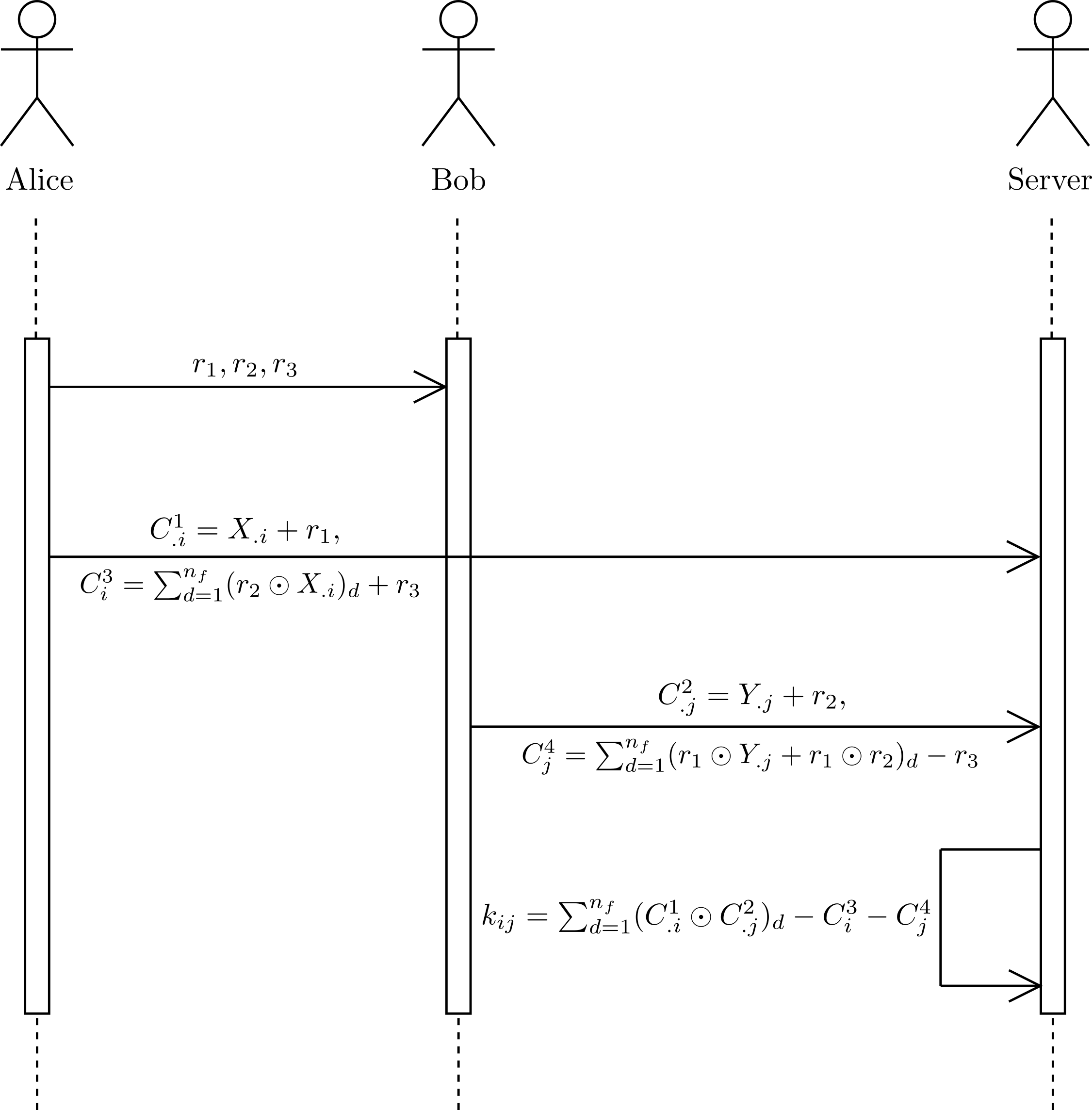}%
  \caption{Overall protocol execution.}
  \label{fig:FlowOfArchitecture}
\end{figure}

After having the complete gram matrix for all samples that Alice and Bob have, the server uses it as a kernel matrix as if it was computed by the linear kernel function on pooled data. Additionally, it is also possible to compute a kernel matrix as if it was computed by the polynomial or radial basis kernel function (RBF) by utilizing the resulting gram matrix. As an example, the calculation of RBF from the gram matrix is as follows.

\begin{equation*} \label{rbfdotproduct}
K(x,y) = \exp\Bigg(-\dfrac{\norm{x \cdot x - 2 x \cdot y + y \cdot y}^2}{2 \sigma^2}\Bigg),
\end{equation*}

where ``$\cdot$'' represents the dot product of vectors, which is possible to obtain from the gram matrix, and $\sigma$ is the parameter utilized to adjust the similarity level. Once the desired kernel matrix is computed, it is possible to train an SVR model by employing the computed kernel matrix to estimate the gaze. In the process of the computation of the dot product, the amount of data transferred among parties is $(n_fn_a + n_fn_b + n_a + n_b + 2n_f) \times d$ bytes where $d$ is the size of one data unit.

\begin{figure*}[h]
    \centering
    \subfloat[Execution time of Alice.]{{\includegraphics[width=0.2\linewidth]{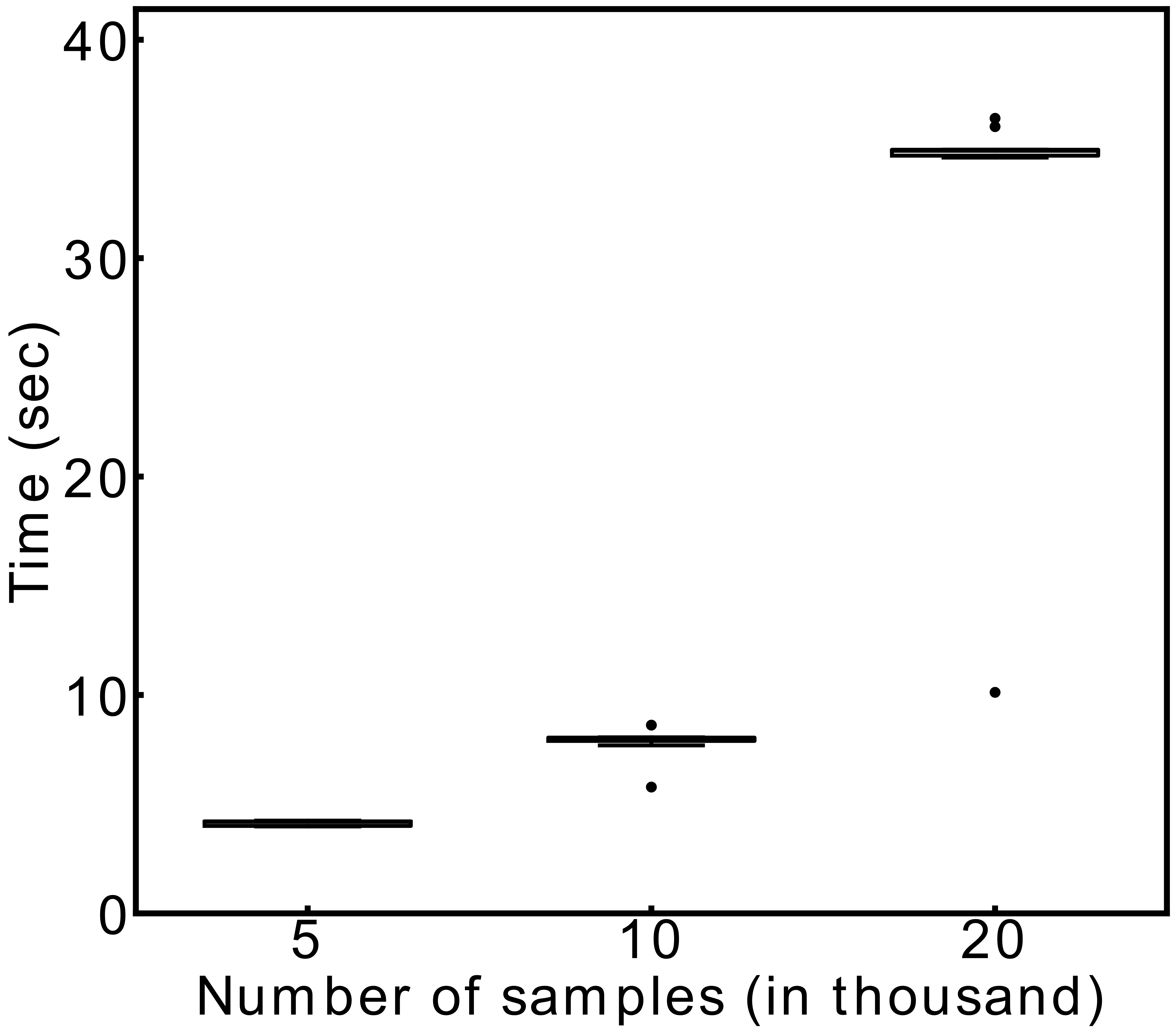}}}
    \hspace{15pt}
    \subfloat[Execution time of Bob.]{{\includegraphics[width=0.2\linewidth]{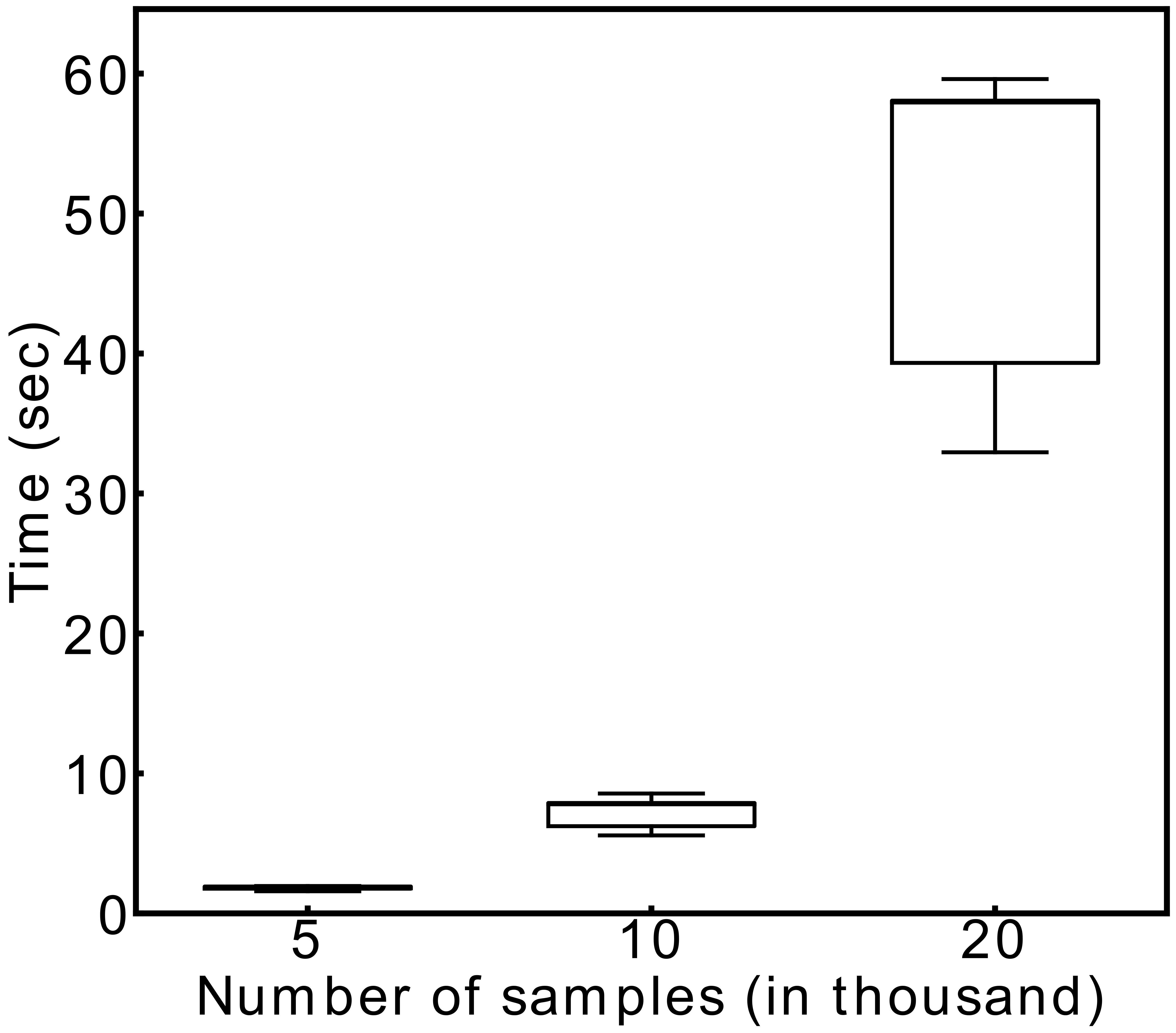}}}
    \hspace{15pt}
    \subfloat[Execution time of Server.]{{\includegraphics[width=0.2\linewidth]{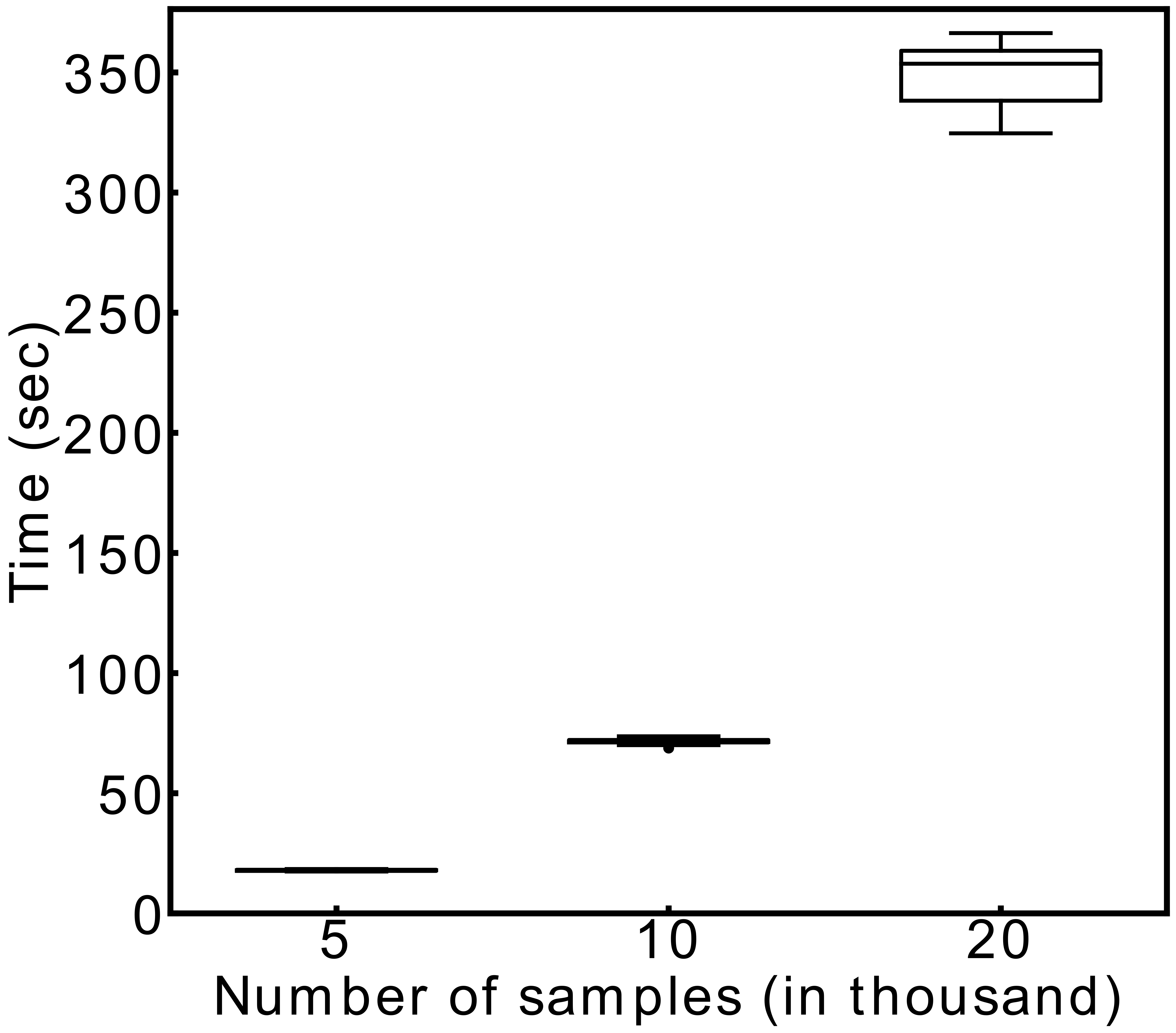}}}
    \hspace{15pt}
    \subfloat[Prediction time of Server.]{{\includegraphics[width=0.2\linewidth]{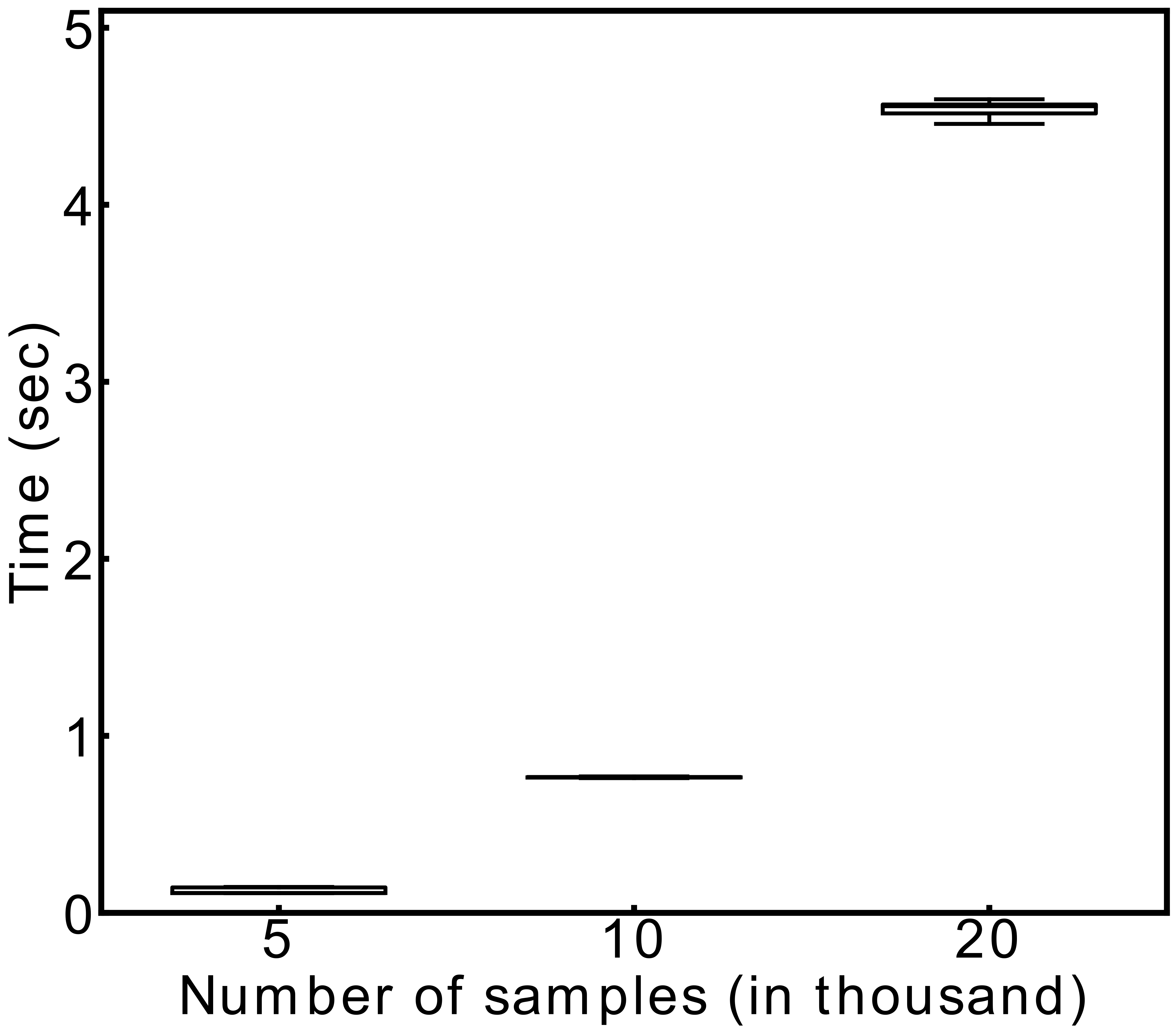}}}
    \caption{The execution time of (a) Alice, (b) Bob and (c) the server are given. We also demonstrate (d) the time required for the prediction of the test samples, which are $20\%$ of the total number of samples in each case.}
    \label{fig:exe_time}
\end{figure*}

\section{Security Analysis}
A semi-honest adversary who corrupts any of the input-parties cannot learn anything about the private inputs of the other input-party. During the protocol execution, two vectors of random values and a single random value are sent from Alice to Bob. The views of the input-parties consist only of vectors with random values. Using these random values, it is not possible for one party to infer something about the other party's private inputs \cite{unal2019framework}.

\begin{theorem}
A malicious adversary $\mathcal{A}$ corrupting the function-party learns nothing more than the result of gram matrix. It is computationally infeasible for $\mathcal{A}$ to infer any information about the input-parties’ data $X$ and $Y$ as long as Perfect RE multiplication is semantically secure (Definition \ref{def:remul}). 
\end{theorem}

\begin{proof}
We first show the correctness of our solution. We assume $n_f = 2$ and encode the function $f_d(x,y) = x_1y_1+x_2y_2$ over some finite ring $\mathsf{R}$ by the following DARE:

\begin{equation*}
\begin{aligned}
    \hat{f_d}(x,y;r) = ( & x_1+r_{11}, y_1+r_{12}, x_2+r_{21}, y_2+r_{22}, \\
    & r_{12}x_1+ r_{22}x_2+r_3, \\ 
    & r_{11}y_1+r_{11}r_{12} + r_{21}y_2+r_{21}r_{22}-r_3)
\end{aligned}
\end{equation*}

Given an encoding $(c_1,c_2,c_3,c_4,c_5,c_6)$, $f_d(x,y)$ is recovered by computing $c_1c_2+c_2c_4+c_5+c_6$.

By the concatenation lemma in \cite{applebaum2017garbled}, we can divide $c_5$ and $c_6$ into $n_f$ shares by using $n_f$ random values instead of a single $r_3$ value.

\begin{equation*}
\begin{aligned}
    \hat{f_d}(x,y;r) = ( & x_1+r_{11}, y_1+r_{12}, r_{12}x_1 + r_{13}, r_{11}y_1+r_{11}r_{12} - r_{13}, \\
    & x_2+r_{21}, y_2+r_{22}, r_{22}x_2+r_{23}, r_{21}y_2+r_{21}r_{22}-r_{23})
\end{aligned}
\end{equation*}

Given an encoding $(c_1,c_2,c_3,c_4,c_5,c_6,c_7,c_8)$,

\begin{equation*}
\begin{aligned}
    \hat{f_m}(x_1,y_1;r) = (c_1,c_2,c_3,c_4)\\
    \hat{f_m}(x_2,y_2;r) = (c_5,c_6,c_7,c_8)
\end{aligned}
\end{equation*}

By the concatenation lemma in \cite{applebaum2017garbled}, $\hat{f_d}(x,y;r) = $ $(\hat{f_m}(x_1,y_1;r), \hat{f_m}(x_2,y_2;r))$ perfectly encodes the function $f_d(x,y)$ if Perfect RE multiplication is semantically secure. 

After showing the correctness, we analyze the security with the simulation paradigm. In the simulation paradigm, there is a simulator who generates the view of a party in the execution. A party’s input and output must be given to the simulator to generate its view. Thus, security is formalized by saying that a party’s view can be simulatable given its input and output and the parties learn nothing more than what they can derive from their input and prescribed output.

The function-party $\mathcal{F}$ does not have any input and output. A simulator $\mathcal{S}$ can generate the views of incoming messages received by $\mathcal{F}$. $\mathcal{S}$ creates four vectors $C^{1'}$,$C^{2'}$,$C^{3'}$,$C^{4'}$ with uniformly distributed random values using a pseudorandom number generator $G'$. Finally, $\mathcal{S}$ outputs $\{C^{1'}$,$C^{2'}$,$C^{3'}$,$C^{4'}\}$. 

In the execution of the protocol $\pi$, $\mathcal{A}$ receives four messages which are masked with uniformly random values generated using a pseudorandom number generator $G$. The view of $\mathcal{A}$ includes $\{C^1$,$C^2$,$C^3$,$C^4\}$. The distribution over $G$ is statistically close to the distribution over $G'$. This implies that

\[\{\mathcal{S}(C^{1'},C^{2'},C^{3'},C^{4'})\} \overset{c}{\equiv} \{view^\pi_{\mathcal{A}}(C^1,C^2,C^3,C^4)\}\]

\end{proof}

\section{Results}
To demonstrate the performance, we conduct experiments on a PC equipped with Intel Core i7-7500U with 2.70 GHz processor and 16 GB memory RAM. We employ varying sizes of eye landmark data, that are $5,000$, $10,000$ and $20,000$ samples of which one-fifth is the test data and we split the data between the input-parties equally. The framework allows us to optimize the parameters of the model in the server without further communicating with the input-parties. Thanks to this, we utilize $5$-fold cross-validation to optimize the parameters, which are the similarity adjustment parameter $\gamma \in \{2^{-3},2^{-2},\cdots,2^{4}\}$ of the Gaussian RBF kernel, the misclassification penalty parameter $C \in \{2^{-3},2^{-2},\cdots,2^{3}\}$, and the tolerance parameter $\epsilon \in \{0.005, 0.01, 0.05, 0.1, 0.5, 1\}$ of SVR. After parameter optimization, we repeat the experiment on varying sizes of eye landmark data with the optimal parameter set $10$ times to assess the execution time. To evaluate the gaze estimation results, we employ mean angular error in the same way as in \cite{Park2018}. Table \ref{tab:mae} demonstrates the relationship between the dataset size and the resulting mean angular error. Since no additional noise is introduced during the computation of the kernel matrix, the results from our privacy-preserving framework are the same with the non-private ones. The mean angular errors are lower compared to the state-of-the-art gaze estimation techniques since we use synthetic data and fixed camera position during image rendering.

\begin{table}[h]
    \caption{The mean angular errors for varying dataset sizes.}
    \label{tab:mae}
    \centering
    \begin{tabular}{ c c }
     \toprule
     \makecell{\# of samples} & \makecell{Mean angular error} \\
     \midrule
     5k & 0.21 \\
     10k & 0.18 \\
     20k & 0.17 \\
     \bottomrule
    \end{tabular}
\end{table}

The amount of time to train and test the models increases as the sample sizes increase since computation requirements get larger. The increment in the dataset size increases the communication cost among parties. The execution times of all parties for $10$ runs with the optimal parameters are shown in Figure \ref{fig:exe_time}. We also demonstrate the amount of time to predict the test samples, which corresponds to one-fifth of the total number of samples to emphasize the real-time working capabilities. In the experiment with $20,000$ samples, for instance, we spend $\approx4.5$ seconds to predict $4,000$ test samples, which corresponds to $1.125$ ms per sample. When the current sampling frequencies of eye trackers are taken into consideration, it is possible to deploy and use the framework to estimate gaze if an optimized communication between the parties is established.

\section{Conclusion}
In this work, we utilized a framework based on randomized encoding to estimate human gaze in a privacy-preserving way and in real-time. Our solution can provide improved gaze estimation if input-parties want to use each other's data for different reasons such as to account for genetic structural differences in the eye region. None of the input-parties has the access to the eye landmark data of the others or the result of the computation in the function party, while the function-party cannot infer anything about the data of the input-parties. Temporal information of the visual scanpath, pupillary, or blinks cannot be reconstructed due to the shuffling of the data, and lack of sensory information and direct access to the eye landmarks. Our solution works in real-time, hence it could be deployed along with HMDs for different use-cases and extended to similar eye tracking related problems if similar amount of features is used. To the best of our knowledge, this is the first work based on function-specific privacy models in the eye tracking domain. The number of parties is a limitation of our solution. Thus, as future work we will extend our work to a larger number of parties.

\bibliographystyle{ACM-Reference-Format}
\bibliography{sample-base}

\end{document}